\documentclass[sigconf]{acmart}
\usepackage{algorithm}
\usepackage[noend]{algpseudocode}
\usepackage{graphicx}
\usepackage{amsmath}
\usepackage{subcaption}
\usepackage{colortbl}
\usepackage{bm}
\usepackage{pifont}
\newcommand{\cmark}{\ding{51}}%
\newcommand{\xmark}{\ding{55}}%
\usepackage{capt-of}
\usepackage{wasysym}
\usepackage{diagbox}
\usepackage{float}
\usepackage{multirow}
\usepackage{threeparttable}
\usepackage{enumitem}
\newcolumntype{A}{>{\centering\arraybackslash}p{0.12\columnwidth}}
\newcolumntype{B}{>{\centering\arraybackslash}p{0.07\columnwidth}}
\newcolumntype{C}{>{\centering\arraybackslash}p{0.04\columnwidth}}
\newcolumntype{X}{>{\centering\arraybackslash}p{0.35\columnwidth}}
\newcolumntype{Y}{>{\centering\arraybackslash}p{0.14\columnwidth}}
\newcolumntype{Z}{>{\centering\arraybackslash}p{0.12\columnwidth}}
\newcolumntype{D}{>{\centering\arraybackslash}p{0.16\columnwidth}}
\newcolumntype{E}{>{\centering\arraybackslash}p{0.28\columnwidth}}
\newcolumntype{F}{>{\centering\arraybackslash}p{0.14\columnwidth}}

\DeclareMathOperator*{\argmin}{arg\,min}

\newcommand{\Ours}{RAGUEL}

\newcommand{\Break}{\State \textbf{break} }
\newcommand{\x}{\mathbf{x}}
\newcommand{\xx}{\mathbf{\widetilde{x}}}
\newcommand{\XX}{\widetilde{X}}
\newcommand{\B}{\mathcal{B}}
\newcommand{\mr}[2]{\multirow{#1}{*}{#2}}
\newcommand{\mrb}[2]{\multirow{#1}{*}{\textbf{#2}}}

\newcommand{\mcb}[2]{\multicolumn{#1}{c}{\textbf{#2}}}

\algnewcommand{\IfThenElse}[3]{
  \State \algorithmicif\ #1\ \algorithmicthen\ #2\ \algorithmicelse\ #3}
\newcommand{\para}[1]{\textit{#1.}}

\theoremstyle{definition}
\newtheorem{definition}{Definition}[section]

\copyrightyear{2022}
\acmYear{2022}
\setcopyright{acmcopyright}\acmConference[CIKM '22]{Proceedings of the 31st ACM International Conference on Information and Knowledge Management}{October 17--21, 2022}{Atlanta, GA, USA}
\acmBooktitle{Proceedings of the 31st ACM International Conference on Information and Knowledge Management (CIKM '22), October 17--21, 2022, Atlanta, GA, USA}
\acmPrice{15.00}
\acmDOI{10.1145/3511808.3557424}
\acmISBN{978-1-4503-9236-5/22/10}

\usepackage{xparse}
\ExplSyntaxOn
\NewDocumentCommand{\tabHeader}{m}
 {
  \seq_set_from_clist:Nn \l_tmpa_seq { #1 }
  \seq_set_map:NNn \l_tmpa_seq \l_tmpa_seq { \textbf{##1} }
  \seq_use:Nn \l_tmpa_seq { & }
 }
\ExplSyntaxOff

\begin{document}

\title{RAGUEL: Recourse-Aware Group Unfairness Elimination}

\author{Aparajita Haldar}
\affiliation{
  \institution{University of Warwick}
  \city{Coventry}
 \country{United Kingdom}
}
\email{aparajita.haldar@warwick.ac.uk}

\author{Teddy Cunningham}
\affiliation{
  \institution{University of Warwick}
 \city{Coventry}
 \country{United Kingdom}
}
\email{teddy.cunningham@warwick.ac.uk}

\author{Hakan Ferhatosmanoglu}
\authornote{Also with Amazon Web Services. This publication presents work performed at the University of Warwick and is not associated with Amazon.}
\affiliation{%
  \institution{University of Warwick}
  \city{Coventry}
  \country{United Kingdom}
}
\email{hakan.f@warwick.ac.uk}

\renewcommand{\shortauthors}{Haldar, Cunningham, and Ferhatosmanoglu}

\begin{abstract}
While machine learning and ranking-based systems are in widespread use for sensitive decision-making processes (e.g., determining job candidates, assigning credit scores), they are rife with concerns over unintended biases in their outcomes, which makes algorithmic fairness (e.g., demographic parity, equal opportunity) an objective of interest. 
`Algorithmic recourse' offers feasible recovery actions to change unwanted outcomes through the modification of attributes. 
We introduce the notion of ranked group-level recourse fairness, and develop a `recourse-aware ranking' solution that satisfies ranked recourse fairness constraints while minimizing the cost of suggested modifications. 
Our solution suggests interventions that can reorder the ranked list of database records and mitigate group-level unfairness; specifically, disproportionate representation of sub-groups and recourse cost imbalance. 
This re-ranking identifies the minimum modifications to data points,
with these attribute modifications weighted according to their ease of recourse.
We then present an efficient block-based extension that enables re-ranking at any granularity (e.g., multiple brackets of bank loan interest rates, multiple pages of search engine results).
Evaluation on real datasets shows that, while existing methods may even exacerbate recourse unfairness, our solution -- {\Ours} -- significantly improves recourse-aware fairness.
{\Ours} outperforms alternatives at improving recourse fairness, through a combined process of counterfactual generation and re-ranking, whilst remaining efficient for large-scale datasets.
\end{abstract}



\maketitle

\section{Introduction}
\label{s:intro}

Machine learning techniques are being used increasingly for a wide range of decision-making.
This includes classification-based decisions, such as medical diagnoses \cite{harper2005review}, judicial verdicts \cite{capuano2014methodology}, financial risk assessments \cite{sarkar2018robust}, as well as ranking-based decisions, such as exam results \cite{valenti2003overview} and job applications \cite{paparrizos2011machine}.
Despite their influence on the real world, automated decision-making systems have come under scrutiny for potentially unfair outcomes between sub-groups separated based on protected data attributes, such as race, gender, and marital status~\cite{hajian2016algorithmic, corbett2017algorithmic,venkatasubramanian2019algorithmic,jin2020mithracoverage,kang2021multifair,bhargava2020limeout}. 
The concept of fairness is applicable more broadly, including technical settings such as fair resource allocation in computer networks~\cite{fairresourceallocation} and fair task assignment in crowdsourcing~\cite{basik2018fair}.

Numerous definitions of algorithmic fairness have gained popularity, each of which has associated bias mitigation strategies, such as ignoring protected attributes in the model and enforcing balanced success/error rates across sub-groups~\cite{mehrabi2021survey,zhang2021omnifair,pitoura2021fairness}. 
However, none of the traditional unfairness mitigation strategies for ranking address the impact of imbalances in `algorithmic recourse'. 
Recourse aims to provide users with a set of feasible actions that can be taken to recover from an unwanted outcome~\cite{ustun2019actionable}.
Recourse fairness implies that the opportunity for, and cost of, improvement or recovery should not favor any sub-group over another.
Even if a model is fair in its expected outputs (e.g., by ensuring equal success rates through demographic parity), failing to consider recourse can lead to issues such as inequity across society 
due to imbalances in the cost of recovery.
`Group-level recourse fairness' (i.e., minimizing the difference in recourse across groups) is thus desirable in many settings, and has been included in recent classification models \cite{gupta2019equalizing, von2020fairness}.
Recourse fairness in ranking remains unexplored despite a range of real-world applications where entities should enjoy comparable costs of recovery to improve their ranking (e.g., job applicants, web search results, online dating profiles, e-commerce product listings~\cite{biega2018equity, singh2018fairness, ghizzawi2019fairank}). 
Recourse-aware ranking is thus beneficial to identify the minimal adjustments that individuals need to make so that there is fairer representation throughout the ranked list.
For example, measures have been prescribed to help low-income individuals demonstrate creditworthiness based on steady income rather than owning credit cards~\cite{bank1997closing}. 
There is a need to identify which of these alternatives is most suitable to offer to a given individual to improve their chances for recourse.

In this paper, we introduce the notion of \textit{ranked} group-level recourse fairness, which can be applied to classification models and ranking-based problems.
For ranking problems, there is a recourse cost to reach some defined `ideal' point that all database members aspire to achieve, whereas for classification problems, it is the cost to reach the classifier boundary, modeled as a hyperplane of target query points that have the desired classifier outcome.

We propose a strategy that ranks records according to each record's (recourse) cost to reach a target point and, given any such ranked list, offers improved ranked recourse fairness with minimal re-ranking.
Thus, we aim to minimize the cost of adjustment while satisfying fairness constraints through our re-ranking strategy,
{\textbf{\Ours}}: \textbf{R}ecourse-\textbf{A}ware \textbf{G}roup \textbf{U}nfairness \textbf{El}imination.  {\Ours} has two steps: i) computing the `ranked group-level recourse fairness ratio' for sub-groups based on a given set of protected attributes (e.g., gender), and ii) adjusting the ranked list such that recourse fairness is improved across sub-groups. 
To achieve the former, {\Ours} determines `counterfactuals' and aggregates the group-wise cost of recourse to reach the counterfactual(s). 
Here, we use the general term `counterfactual' to denote any target point into which a record can feasibly be transformed.
For the latter, {\Ours} first identifies the disadvantaged data points that require minimal perturbation towards their counterfactuals. 
{\Ours} then iteratively adjusts the attributes of a point until its new rank satisfies fair representation and recourse fairness at the group-level.
{\Ours} recommends these changes to clients as potential recourse steps.
Therefore, our solution considers recourse during the point selection process as well as the re-ranking process to minimize the cost of the recommended interventions. 
Unlike many existing strategies~\cite{gupta2019equalizing,salimi2019interventional,salimi2020database}, {\Ours} does not require re-weighting or retraining the model after database repair.



We then extend {\Ours} to handle multiple ranked `blocks', which occurs when multiple batches or brackets are used in ranking or classification outcomes.
For example, 
loan application screening may require multiple different acceptance thresholds corresponding to different interest rate brackets offered. 
Similarly, information retrieval and recommendation systems typically present results in batches, such as search engine results pages.
For such situations, one can model each bracket/batch to have its own threshold boundary, and divide the ranking into multiple blocks to match these boundaries.
In the block-based approach, points are re-ranked to improve proportional representation and ensure ranked group-level recourse fairness of sub-groups within every block. 

The contributions of {\Ours} are summarized as follows.

\para{1. Ranked Group-Level Recourse Fairness}
We introduce `ranked group-level recourse fairness', which measures fairness in recourse actions in ranked lists; alongside the need for fair representation in ranking, this forms the basis for our problem. 
In contrast to existing fair ranking approaches, {\Ours} i) provides a set of feasible actions to the \textit{user} (e.g., loan applicant) that would result in the desired outcome, 
and ii) enables the \textit{platform} (e.g., bank) 
to consider recourse fairness via minimal cost of opportunities presented to disadvantaged sub-groups.

\para{2. Recourse-Aware Fair Ranking}
{\Ours}'s iterative and minimally invasive approach considers the cost associated with the re-ranking while improving recourse fairness. The experiments show an improvement on recourse fairness by
15\% compared to the initial ranking, by up to 200\% compared to traditional fair classifiers (which actually exacerbate recourse unfairness due to the differences in objectives), and by 37\% compared to the fair ranking method FA*IR~\cite{zehlike2017fa}.

\para{3. Efficient Counterfactual Generation and Re-Ranking}
When multiple counterfactual points exist (e.g., when there is a classifier decision boundary), our inverse classification-based solution is around 10x faster on large datasets compared to the baseline~\cite{ustun2019actionable}.
{\Ours} also re-ranks large datasets with ease, whereas FA*IR~\cite{zehlike2017fa} and FoEiR~\cite{singh2018fairness} could not handle more than 400 and 50 records respectively in practice.

\para{4. Block-Based Solution}
To the best of our knowledge, no prior ranking solution adjusts the granularity of the fairness measure. The proposed block-based re-ranking better preserves similarity to the original list compared to the non-block version.
Besides practical use cases, {\Ours}-Block shows significant performance benefits.
It requires fewer modifications to the ranked list with a 1.3x lower cost for achieving recourse fairness, and it is up to 30x faster in unfairness mitigation in our experiments. 
The solution is shown to be scalable for datasets with millions of records.


\section{Related Work}
\label{s:related-work}

We now review recent literature that relates to our problem.
Table~\ref{tab:related-work} summarizes the main limitations of the most relevant work.

\begin{table}[t]
\small
\addtolength{\tabcolsep}{-0.02\columnwidth}
\centering
\caption{Summary of limitations in related work}
\label{tab:related-work}
\begin{tabular}{XZZZZY}
\toprule
\mrb{2}{Desideratum} 
& \textbf{MACE~} & \textbf{AR~ } & \textbf{FA*IR~ } & \textbf{FoEiR~} & \mrb{2}{{\Ours}} \\
& \textbf{\cite{karimi2020model}} & \textbf{\cite{ustun2019actionable}} & \textbf{\cite{zehlike2017fa}} & \textbf{\cite{singh2018fairness}} & \\
\midrule
Counterfactual generation   & \cmark & \cmark & \xmark & \xmark & \cmark \\
Distance-agnostic   & \cmark & \cmark & \xmark & \cmark & \cmark \\
Fair ranking   & \xmark & \xmark & \cmark & \cmark & \cmark \\
Recourse consideration   & \xmark & \cmark & \xmark & \xmark & \cmark \\
Adjustable granularity   & \xmark & \xmark & \xmark & \xmark & \cmark \\
\bottomrule
\end{tabular}
\vspace{-0.5cm}
\end{table}


There has been an increased focus on understanding whether, and how, inherent biases result in decision-making systems being unfair to individuals or groups~\cite{barocas2017fairness,corbett2018measure,mehrabi2021survey}. 
There have been several unfairness mitigation measures, such as simply omitting sensitive attributes and pre-processing the data to mask such attributes~\cite{kamiran2009classifying,feldman2015certifying}.
Counterfactual fairness and explicit causal models capture the intuition that a decision outcome should not be influenced by sensitive attributes~\cite{kusner2017counterfactual,zhang2018fairness,wang2019repairing}. 
Interventions are used to reassign specific values to attributes of points to generate counterfactual explanations of models by asking ``what if things had been different?''. 
The term `counterfactual' here indicates a point having different attribute values, leading to a different model prediction. 
LIME~\cite{ribeiro2016should} offers local explanations for individual predictions of black-box models. 
CLEAR~\cite{white2019measurable} extends this to compute the nearest counterfactuals and measure their fidelity to the underlying model. 
MACE~\cite{karimi2020model} generates plausible and diverse nearest counterfactuals in a model-agnostic manner.

`Recourse' is defined 
as an individual's ability to change their outcome by altering their attributes, in a recovery process that is similar to the interventions to generate counterfactuals~\cite{ustun2019actionable}. 
However, most counterfactual generation methods do not incorporate the cost (financial or otherwise) of these recovery actions.
\citet{ustun2019actionable} focus on the viability of providing recourse to individuals based on the interventions required and the effect of immutable attributes. 
Their study has inspired a number of strategies for equalizing recourse between sub-groups in classification models. 
One approach re-weights groups with large recourse costs relative to groups with small recourse costs~\cite{gupta2019equalizing}. 
Another method applies causal-based ``equality of effort''~\cite{huan2020fairness} for potential ``discrimination removal'' by adding new optimization constraints to the classifier.
\citet{karimi2020algorithmic} explore probabilistic methods to determine recourse and counterfactuals given limited causal knowledge.

Fairness in ranked lists has attracted attention with a focus on mitigating the position bias that leads to unfair representation in ranking~\cite[e.g.,][]{zehlike2017fa, biega2018equity, singh2018fairness, ghizzawi2019fairank,mehrotra2018towards}.
FA*IR~\cite{zehlike2017fa} is a top-$k$ ranking algorithm that aims to ensure that the proportion of protected individuals in every subset of a top-$k$ ranking remains above some minimum threshold while maintaining the utility of the ranking.
\citet{singh2018fairness} compare exposure allocation with query relevance, motivated by different fairness constraints, 
to examine the utility trade-offs. 
Similarly, ``equity of attention'' aims to optimize fairness versus utility trade-offs by amortizing fairness accumulated across a series of rankings~\cite{biega2018equity}. 

We introduce the notion of recourse fairness for ranked lists by imposing a constraint to ensure `ranked group-level recourse fairness'. 
%
Current methods (e.g., FA*IR~\cite{zehlike2017fa}, FoEiR~\cite{singh2018fairness}) consider utility trade-offs without incorporating the costs incurred during re-ranking. 
We also introduce a block-based solution for recourse-aware re-ranking. 
This approach simultaneously processes multiple blocks, each with a different acceptance threshold.

\section{Fair Recourse-Aware Ranking}
\label{s:methodology}
Here, we introduce necessary notation and define the fairness constraints for our problem before outlining our re-ranking solution {\Ours}, 
its block-based variant, and extensions to our work.

\subsection{Preliminaries}
\label{ss:preliminaries}
We define a database $\mathcal{D}$, with $D$ records, in which each record is represented as a vector $\x \in \mathbb{R}^m$ with $m$ attributes such that $\x = \{x_1, x_2, \dots, x_m\}$.
%
A fundamental part of our setting is the notion of recourse cost, which is 
the difficulty of changing a prediction by taking feasible actions to alter attribute values~\cite{ustun2019actionable}.
When we extend this to ranked lists, recourse cost is the difficulty of the actions to reach some ideal point or hyperplane (e.g., the top of the ranking).
This chosen target is the counterfactual point $\x'$, calculated as: $\x' = \argmin_{\x^* \in \mathcal{X}^*} d(\x, \x^*)$. 
Here, $d(\x,\x^*)$ denotes the distance function that quantifies the cost of the actions required to transfer from $\x$ to $\x^*$.
That is, given a set of candidate counterfactual points $\mathcal{X}^*$, the counterfactual point to $\x$ is the point $\x' \in \mathcal{X}^*$ that has the minimum distance to $\x$.

For settings involving a classifier, we generate counterfactuals by considering the minimal actions to reach the decision boundary (e.g., negatively classified points trying to change their outcome).
We use $f:\mathbb{R}^m \rightarrow \{-1,1\}$ to denote the decision-making model that classifies records into two classes.
This assumption is without loss of generality since any multi-class classifier can be considered as a stack of several binary classifiers.
In this setting, $\mathcal{X}^*$ is the set of all candidate counterfactual points lying on the classifier boundary. 
For general ranking problems, we assume a single counterfactual point -- the point at the top of the ranking -- therefore we have $\mathcal{X}^* = \{\x'\}$.
This can be generalized in future work by determining recourse with respect to multiple possible counterfactual points.

The recourse cost $c(\x)$ is defined as the cost to reach $\x'$ from $\x$.
That is, $c(\x) = d(\x,\x')$.
In Example 1 (Section~\ref{ss:examples}), the `Cost' column of Table~\ref{tab:example} reflects these recourse values based on the weighted Euclidean distance to the corresponding counterfactual on the boundary line, as seen in Figure~\ref{fig:example}.
In our method for increasing fairness in representation and recourse, a chosen point, $\x$, is modified towards its counterfactual point, $\x'$, and $\xx$ denotes this modified point.

The database (or any subset thereof) can be divided into sub-groups, denoted as $S_j$, according to some protected attribute (e.g., gender, race, marital status) for which we are interested in evaluating fairness of representation and recourse.
Without loss of generality and as is common practice in literature~\cite{pitoura2021fairness,huan2020fairness}, we illustrate our approach for two sub-groups, $S_1$ and $S_2$, and they are composed such that $|S_1 \cup S_2| = D$ and $|S_1 \cap S_2| = 0$.
In any pair of sub-groups, where one sub-group contains records with a protected attribute, $p$ denotes the proportion of the global database with this protected characteristic.
That is, $p = |S_2|/D$, assuming $S_2$ contains the records associated with the protected characteristic. 
This represents the ideal proportion of protected records in any subset of the database.

$X_k$ denotes some ordered subset of $\mathcal{D}$ containing $k\!\leq\!D$ points and the
proportion of $X_k$ that consists of records with the protected attribute is $p_k$.
%
Finally, in our block-based approach, we segment the database into $B$ blocks, with each block $b$ as part of the set $\mathcal{B}$.


\subsection{Fairness}
\label{ss:fairness}
We first introduce the simplest definition of fairness given two sub-groups, which ensures that their proportional representation in any subset correctly reflects that in the database.

\begin{definition}{\textbf{Fair Representation}}
\label{def:repr}
Any subset $X \subseteq \mathcal{D}$ represents the protected sub-group fairly if, given some tolerance $\epsilon$, $|p_X - p| \leq \epsilon$, where $p_X$ is the proportion of protected records in $X$. 
\end{definition}

Given an ordered subset $X_k$, we can extend this definition to recursively hold for ordered subsets of the database.

\begin{definition}{\textbf{Ranked Group-level Fair Representation}}
\label{def:rankfair}
An ordered subset $X_k$ satisfies ranked group-level fair representation if, for every $X_i$ (where $1 \leq i < k$, $k \geq 2$), it satisfies fair representation. Any singleton set $X_1$ is considered un-ranked and always fair.
\end{definition}

Definition~\ref{def:rankfair} is a reformulation of the sufficient condition for ranked group fairness in FA*IR~\cite{zehlike2017fa}.
However, while this definition aims to achieve balanced representation for every subset $X_i$, it does not take the cost of recourse into account.
Therefore, we propose `ranked group-level recourse fairness' as a constraint in the ranking.

To measure ranked fairness with respect to recourse, we first formalize the notion of recourse fairness between two sub-groups.
The `group-level recourse fairness ratio', $r$, assesses how balanced two sub-groups are in terms of the average cost of recovery actions:
\begin{equation}
\textstyle
  \label{eq:af}
    r = \frac{\min(M_1, M_2)}{\max(M_1,M_2)} \qquad 0 \leq r \leq 1
\end{equation}
where $M_1$ and $M_2$ are the mean recourse costs for $S_1$ and $S_2$ respectively, defined as: $M_j = \frac{1}{|S_j|}\sum_{\x_i \in S_j}{c(\x_i)}$.
Near-ideal fairness is represented by $r$-values close to 1, as the mean recourse costs for both sub-groups are similar, 
whereas $r$-values close to zero show that one sub-group is severely disadvantaged when trying to recover from an unwanted outcome.
%
%

%
%
\begin{definition}{\textbf{Group-Level Recourse Fairness}}
\label{def:group-recourse-fair}
Any subset $X \subseteq \mathcal{D}$ satisfies group-level recourse fairness if $r_i \geq 1 - \phi$.
\end{definition}

Recourse fairness therefore ensures that the two sub-groups maintain a ratio of their mean recourse costs that is within some tolerance $\phi$ (distinct from $\epsilon$).
This may be extended to an ordered subset $X_k$ in a manner similar to Definition~\ref{def:rankfair}, as described below.
\begin{definition}{\textbf{Ranked Group-Level Recourse Fairness}}
\label{def:actionfair}
An ordered subset $X_k$ satisfies ranked group-level recourse fairness if $r_i \geq 1 - \phi$, for every $X_i$ (where $1 \leq i < k$). 
For any singleton set $X_1$, we always have $r_1 = 1$.
\end{definition}
\noindent Definition~\ref{def:actionfair} acts as a constraint for balancing recourse, in addition to Definition~\ref{def:rankfair}, which balances demographic parity.

\subsection{Problem Statement}
\label{ss:problem-statment}
Given a dataset $\mathcal{D}$, we first obtain the ordered set $X_D$ by ranking records according to their recourse costs to reach the chosen counterfactual points. 
The goal is to identify the actions needed to construct an ordered set of modified points $\XX_D$ such that $\XX_D$ satisfies Definitions~\ref{def:rankfair} and \ref{def:actionfair}, whilst ensuring that we minimize the extent of the modification, i.e., $\min \frac{1}{D}\sum_i d(\x_i, \xx_i)$. 


\subsection{Examples}
\label{ss:examples}
We illustrate our solution in the context of loan applications.
We also discuss how {\Ours} can be applied to ranking-based problems such as fair webpage ranking.

\begin{table}[t]
\begin{threeparttable}
    \captionof{table}[t]{Original and re-ranked example data points}
    \label{tab:example}
    \setlength{\tabcolsep}{0.25\tabcolsep}
    \begin{tabular}{AC|BBBB|BB|BBBB}
        \toprule
        \multicolumn{2}{D|}{\textbf{Name and}} & \multicolumn{4}{E|}{\textbf{Before Re-Ranking}} & \multicolumn{2}{F|}{\textbf{Counter-factual}}& \multicolumn{4}{E}{\textbf{After Re-Ranking}}\\
        \cmidrule{3-12}
        \multicolumn{2}{D|}{\textbf{Gender}} & \textbf{LA} & \textbf{LD} & \textbf{Cost} &  \textbf{Rk.} & \textbf{LA} & \textbf{LD} & \textbf{LA} & \textbf{LD} &  \textbf{Cost} & \textbf{Rk.}\\
        \midrule
        Abdul  & M  & 3.5 & 6 & 0.33 & 1 & 3.06 & 6.11 & 3.5 & 6 & 0.33 & 1 \\
        Bogdan & M  & 2   & 1 & 1.00 & 2 & 0.67 & 1.33 & 2   & 1 & 1.00 & 3 \\
        Chiara & F+ & 4  & 4 & 1.33 & 3 & 2.22 & 4.44 & \textbf{3.45} & \textbf{4} & \textbf{0.97} & \textbf{2} \\
        Diana  & F+  & 5   & 4 & 2.00 & 4 & 2.33 & 4.67 & 5   & 4 & 2.00 & 4 \\
        \bottomrule
        \end{tabular}
        \begin{tablenotes}
      \footnotesize
      \item M = male, F+ = female and all other gender identities, LA = loan amount in \$'000, LD = loan duration in years, Cost = recourse cost, Rk. = rank 
    \end{tablenotes}
  \end{threeparttable}
\end{table}



\para{\textbf{1. Recourse-aware Loan Evaluations}}
Consider four loan applicants who have been rejected and placed onto a ranked waiting list. 
Abdul (rank \#1) and Bogdan (\#2) identify as male whereas Chiara (\#3) and Diana (\#4) identify differently.
Table~\ref{tab:example} and Figure \ref{fig:example} show this ranked list and the subsequent recourse-fair re-ranking after applying {\Ours}. 
In this example, the recourse cost is the cost for negatively classified points to reach their counterfactual points on the classifier boundary (i.e., the loan being granted) considering only the loan amount (LA) and loan duration (LD). 
Each applicant (\textit{user}) is capable of specific actionable changes, such as changing the requested loan amount in \$50 increments, or changing the duration of the requested loan in one year increments.
Initially, despite fair representation, the average cost of recourse is higher for non-male customers, which may be due to them facing greater barriers towards gaining credit (as previously claimed in the real world \cite[e.g., ][]{Garrison1976, Ladd1982, Alesina2013}).
As a result, the bank (\textit{platform}) may be interested in making conditional loan offers to lessen the extent of this disparity, subject to regulations. 
The platform's goal here would be to re-rank the list to achieve recourse fairness via minimal interventions (e.g., for widening participation of disadvantaged customers, or for
inclusive representation due to policies and regulations).
Since the top of the ranking shows disproportionate under-representation of female customers, 
{\Ours} identifies potential actions for Chiara (e.g., providing proof of her ability to cover \$550 from other sources) that would move her up in the ranked list (e.g., a better chance of acceptance with a slightly smaller loan amount offered).
In doing so, the bank improves ranked recourse fairness by enabling the disadvantaged group to perform actions to meet their conditional acceptance criteria.
An alternative outcome is that the bank chooses to offer loans to applicants from the list (e.g., due to an availability of more funds), where candidates higher up in the list (now including Chiara) would secure their requested loans. 
{\Ours} provides a mechanism for both fair representation and fair recourse for all sub-groups with minimal re-ranking.
\begin{figure}[tb]
\centering
 \includegraphics[width = 0.9\columnwidth]{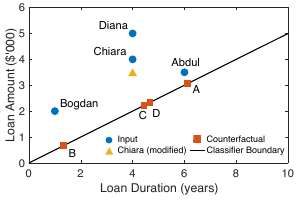}
 \vspace{-0.3cm}
  \captionof{figure}{Example data points plotted}
  \label{fig:example}
\end{figure}

To utilize the block-based solution, the bank can generate equally-sized `blocks' containing two individuals each, with individuals from the first block (Abdul and Bogdan) being offered a lower interest rate. 
Chiara might be offered conditional acceptance at the lower rate, provided she requests a shorter loan amount, moving her into the first block. 
If the bank has blocks containing four individuals, the block would already appear balanced with respect to gender, and no recourse would need to be offered.

\para{\textbf{2. Recourse-Aware Webpage Rankings}}
{\Ours} can be applied to any ranked list, such as top-$k$ query results. 
Consider a set of web-pages (\textit{users}), all of which would like to be listed on the first page of the search engine (\textit{platform}) results (i.e., the top-10 web-pages).
Unlike the case of loan applications, there is no classifier boundary, and the counterfactual is the query itself.
Webpages are ranked by their recourse costs, derived simply based on the difficulty of reaching the counterfactual.
It is in the search engine's best interests to strive for ``fair'' ranking across all categories of web-pages (e.g., business, education, entertainment), while the organizations running the web-pages would like to know what recourse options they have to move up in the ranking (e.g., to help guide their search engine optimization). 
Considering recourse costs helps to identify the interventions (e.g., mobile-friendly formatting, faster page loading, adding hyperlinks) that can enable re-ordering with minimal cost.
If the search engine were only concerned with achieving recourse fairness on each page of results, the less granular block-based solution can be applied, with each block representing a page of search engine results.

\subsection{Methodology}
\label{ss:methodology}
{\Ours} produces a `recourse-aware ranking' that corrects group-level imbalances in terms of recourse and representation.
The first step computes the recourse cost for each record to reach the relevant counterfactual point.
Second, the records are ranked, in ascending order, according to their recourse costs.
{\Ours} then identifies the minimal interventions applied to points such that the re-ranked list mitigates (or eliminates) any (group-level) unfairness.
{\Ours} can also be extended to a block-based approach in which the ranking is subdivided into a finite number of blocks.

\subsubsection{Identifying Counterfactuals and Recourse Costs}
\label{sss:counterfactuals}
In the classification setting, such as Example 1, we treat the decision boundary as a query and use integer programming-based inverse classification~\cite{aggarwal2010inverse,lash2017generalized} to identify the feasible actions for a point to reach the boundary and flip its classification outcome.
Scenarios where decisions are influenced by multiple factors (e.g., where the classification boundary is a hyperplane) require non-trivial solutions to efficiently find the counterfactual points for large datasets.
By using generalized inverse classification, it is possible to include bounds on the allowable interventions, add a sparsity constraint that ensures fewer attributes are changed, and support non-linearity in recourse weights~\cite{lash2017generalized,laugel2018comparison}. 
Minimizing the cost of such interventions thereby produces the counterfactual point.
For the general case of ranking, such as Example 2, we identify the interventions needed for each record to reach the top of the ranking, the nearest representative, or other appropriately defined counterfactual points.

We use weighted Euclidean distance in our experiments.
This assumption is useful where no information exists about the classifier or ranking process, and only the `recourse weight' for each attribute is known. 
Thereby, utility can be maintained by minimizing the cost of recourse interventions.
Any other function for $d(\x, \x')$ may be used, if available.
Hence, for each $\x_i\!\in\!\mathcal{D}$, our counterfactual $\x_i'$ is computed as:
\begin{equation}
    \textstyle
    \x_i' = \argmin_{\x_j^* \in \mathcal{X}^*} \sqrt{\sum_{k=1}^{m} w_k |x^*_{j,k}-x_{i,k}|^2}
\end{equation}
The recourse cost $c(\x_i)$ is thus $c(\x_i) = d(\x_i, \x_i')$.
%
Recourse weights ($w_k$) can be user-defined, assigned by experts, or learned from data.
These weights are customizable to reflect the ease with which certain actions can be taken, and should be normalized. 
In Example~1, decreasing the requested loan amount ($w_{LA} = 0.5$) is easier than increasing the duration of the loan ($w_{LD} = 1$).
One cannot change immutable attributes (e.g., race) or conditionally immutable attributes (e.g., marital status), which is reflected in the weights.




\subsubsection{Recourse-Aware Re-Ranking}
\label{ss:re-rank}
The next step is to apply interventions to ensure fairness in the ranked list. 
Definition~\ref{def:actionfair} helps to identify minimal modifications that can re-rank the records towards more balanced group-level recourse costs alongside other fairness constraints. 
In this way, {\Ours} improves group-level recourse fairness by suggesting actions that expend minimal resources. 

The steps are outlined in Algorithm~\ref{alg:recourse-aware-rerank}. 
We start by defining $\XX$, which grows one record at a time and will eventually represent the ranked database of modified points (Lines 2--3).
If $X_i$ satisfies Definition~\ref{def:rankfair}, $\x_i$ is added to $\XX$ (i.e., we maintain the status quo)  (Lines 4--5). 
Otherwise, as the addition of $\x_i$ would result in $\XX$ no longer satisfying ranked fair representation, we consider substituting $\x_i$ for a lower ranked point.
Substitution is possible if interventions on a lower ranked point decrease its recourse cost such that it gains a higher rank than $\x_i$.
Hence, the algorithm iterates through the remaining records and, if Definition \ref{def:rankfair} is satisfied by one of these points $\x_j$, it proceeds to modification (Lines 6-8).
Modification occurs by iteratively and minimally modifying attribute values of $\x_j$ into $\xx_j$ until Definitions~\ref{def:rankfair} and \ref{def:actionfair} are satisfied by the resulting new ordered subset of points (Line 9-11).

Attributes are modified in order of their recourse weights to prioritize interventions with lowest recourse weights.
This is because sparsity of the recourse vector can make recourse options more understandable to users~\cite{laugel2018comparison}.
If no amount of modification leads to the desired ranking, the attribute is reset to its original value and the next feasible attribute is modified, and so on, followed by combinations of attributes in the same order, if necessary.
That is, if three attributes (A, B, C) exist with $w_A < w_B < w_C$, modifications are attempted in the following order: A, B, C, A\&B, A\&C, B\&C, and A\&B\&C. 
In Example 1, updating Chiara's loan amount in \$50 increments leads to the given re-ranking with minimal cost of intervention.
In the unlikely event that no re-ranking is possible at any stage (e.g., when tolerance margins are tight), the algorithm fails to satisfy the fairness constraints and copies the remaining ranked list as is. This exit strategy is rarely needed (on average, 10\% of the blocks at the strictest tolerance setting in our experiments require it).
At any stage during re-ranking, we expect at least $p-\epsilon$ points to be available to be re-ranked since this proportion was successfully maintained by the fairness constraint so far. 

\begin{algorithm}[t]
\small
\algrenewcommand\algorithmicindent{0.75em}%
\begin{algorithmic}[1]
\Function{Re-ranking}{$\mathcal{D}$}
\State $\XX \gets$ \O
    \For{$1 \leq i < D$}
        \If{$X_i \not\models$ Definition~\ref{def:rankfair}}
            \State $\XX \gets \XX \cup \x_i$
        \Else
            \For{$i+1 \leq j \leq D$}
                \If{$\XX \cup \x_{j} \models$ Definition~\ref{def:rankfair}}
                    \State Modify $\x_{j}$ into $\xx_j$ until $\XX \cup \xx_{j} \models$ Definition~\ref{def:actionfair}
                    \State $\XX \gets \XX \cup \xx_j$
                    \Break
                \EndIf
            \EndFor
        \EndIf
    \EndFor
\EndFunction
\end{algorithmic}
\caption{Recourse-Aware Re-Ranking}
\label{alg:recourse-aware-rerank}
\end{algorithm}



\subsubsection{Block-Based Recourse-Aware Re-Ranking}
\label{sss:block-based-rerank}
Until now, {\Ours} has focused on ranked lists with recourse costs relative to a \textit{single} query (e.g., classifier boundary, top-$k$ query). 
There are applications where the ranking needs to be done with respect to \textit{multiple} queries that have an inherent ranking (e.g., interest rate brackets in Example 1, search engine result pages in Example 2). 
To enable diversity in the set of records that are offered recourse interventions, it is useful to consider fairness within groups of points that are similar/clustered with respect to a classification model or search query, by considering these as independent blocks for re-ranking.
Also, it is often not possible to modify an attribute by the precise amount needed to re-rank it such that Definition~\ref{def:rankfair} is satisfied. 
In these cases, a block-based approach is more suitable as it can handle multiple queries, compare clusters separately, and perform less granular re-ranking.

Algorithm~\ref{alg:block-based} outlines the block-based re-ranking process that handles these cases.
As a pre-processing step, we subdivide the ranking into a set of blocks, $\mathcal{B}$.
If there are $B$ blocks each with $n$ records, $b_1$ contains the first $n$ records, $b_2$ contains the records with ranks $n\!+\!1$ to $2n$, etc.
In our experiments, the number of blocks determines the number of points in each block, which is kept constant throughout re-ranking.
We assume that recourse costs within a block (and thus the ranking order) are close enough to one another that we can focus on satisfying our fairness conditions at the less-granular block-level.
And so, although irregularly sized blocks (e.g., variable-width histogram bins, clusters) can be used, the recourse cost range within each block should be minimized to ensure high accuracy and fairness.


If $b_i$ does not satisfy fair representation (Definition~\ref{def:repr}), its lowest ranked point block is moved into $b_{i+1}$ (Lines 1--4).
The highest ranking feasible point from $b_{i+1}$ replaces it, with interventions performed to modify its ranking (Lines 4--6).

This overall approach for {\Ours}-Block is similar to Algorithm~\ref{alg:recourse-aware-rerank}, with an additional constraint, which is:
\begin{equation}
\label{eq:constraint}
    \begin{split}
    c(\xx) \leq \beta & \qquad \text{if } M_1 \leq M_2 \\
    c(\xx) \geq \beta & \qquad \text{if } M_1 \geq M_2 
\end{split}
\end{equation}
where $\bar{\x}$ is removed from $b$ and the bound, $\beta$, is defined as:
\begin{equation*}
\textstyle
    \beta = (M_1|S_1|-c(\bar{\x}))\frac{M_2}{M_1}\frac{|S_2|+1}{|S_1|-1} - M_2|S_2|
\end{equation*}
This constraint guarantees that the interventions used to improve fair representation also improve recourse fairness.

\begin{theorem}
\label{th:bound}
Minimal modification of a data point that results in improved fair representation of the block is guaranteed to improve recourse fairness of the block if the constraint in Equation \ref{eq:constraint} is met.
\end{theorem}

\begin{proof}
Consider data points $\x_i$ in a block of size $n$.
Each $\x_i \in b$ has a recourse cost $c(\x_i)$ to its respective target counterfactual point $\x_i'$.
Assume the block, $b_i$, is divided into subgroups $S_1$ and $S_2$ based on some sensitive attribute, with $S_2$ being under-represented (i.e., $|S_2| < |S_1|$). 
Within $b$, the mean costs are:
\begin{equation*}
\textstyle
    M_1 = \frac{1}{|S_1|}\sum_{\x_i \in S_1}{c(\x_i)} \qquad M_2 = \frac{1}{|S_2|}\sum_{\x_i \in S_2}{c(\x_i)}
\end{equation*}
Then, with the modified point $\xx$ (of the under-represented sub-group) from $b_{i+1}$, we aim to improve the proportion $\tfrac{|S_2|}{n}$ as per Definition~\ref{def:repr}. 
The block size is maintained, thus the lowest ranked point $\bar{\x}$ (of the over-represented sub-group) is pushed into the next block, $b_{i+1}$. 
The new mean total costs are:
\begin{equation*}
\textstyle
    M_1' = \frac{-c(\bar{\x}) + \sum_{i \in S_1}c(\x_i)}{|S_1|-1} \qquad
    M_2' = \frac{c(\xx) + \sum_{i \in S_2}c(\x_i)}{|S_2|+1}
\end{equation*}
%
In the trivial case where $|S_1| = n$ and $|S_2| = 0$, we know that $M_2 = 0$ and any choice of $\xx$ will improve the ratio $\frac{M_2'}{M_1'}$. 
In cases where $|S_1| > |S_2| > 0$, we have two situations: $\bar{\x} \in S_1$ or $\bar{\x} \in S_2$. 
The latter is not a permitted perturbation as it makes no improvement to fair representation. 
When $\bar{\x} \in S_1$, by substituting $M_1$ and $M_2$ into the expression for $\frac{M_2'}{M_1'}$, we obtain:
\begin{align*}
\frac{M_2'}{M_1'} &= 
    \frac{c(\xx) + \sum_{i \in S_2} c(\x_i)}
    {- c(\bar{\x}) + \sum_{i \in S_1} c(\x_i)}
    \times
    \frac{|S_1|-1}{|S_2|+1} \\
&=
\frac{M_2|S_2| + c(\xx)}{M_1|S_1|-c(\bar{\x})} \times \frac{|S_1|-1}{|S_2|+1}
\end{align*}
\begin{flalign*}
\textstyle
    \text{Thus, } && c(\xx) &= (M_1|S_1|-c(\bar{\x}))\frac{M_2'}{M_1'}\frac{|S_2|+1}{|S_1|-1} - M_2|S_2| & 
\end{flalign*}
Consider the bound given by: $\beta = (M_1|S_1|-c(\bar{\x}))\frac{M_2}{M_1}\frac{|S_2|+1}{|S_1|-1} - M_2|S_2|$.
If $M_1 \leq M_2$ and $c(\xx) \leq \beta$, it follows that $\frac{M_2'}{M_1'} < \frac{M_2}{M_1}$. 
Similarly, if $M_1 \geq M_2$ and $c(\xx) \geq \beta$, it follows that $\frac{M_2'}{M_1'} > \frac{M_2}{M_1}$. 
In either case, the new ratio $r' = \frac{\min(M_1', M_2')}{\max(M_1',M_2')}$ moves closer to 1 and thus recourse fairness improves within the block.
\end{proof}

Note that points can only move up/down by one block when they are re-ranked.
This restricts the permitted cost of modification, and means that the exit strategy (introduced in Section \ref{ss:re-rank}) is more relevant in {\Ours}-Block. 
In situations where no feasible modification is possible, the block is copied as is and we proceed to the next block (illustrated in Figure \ref{fig:married-plots} and Table \ref{tab:effectiveness}).

\begin{algorithm}[t]
\small
\algrenewcommand\algorithmicindent{0.75em}%
\begin{algorithmic}[1]
\Function{BlockBased}{$\mathcal{D}$, $\B$}
    \For{$1 \leq i < B$}
        \While{$b_{i} \not\models$ Definition~\ref{def:repr}}
        	\State Move last element of $b_{i}$ into $b_{i+1}$
            \State Find highest ranked $\x \in b_{i+1}$ s.t. $b_{i} \cup \x \models$ Definition~\ref{def:repr}
            \Repeat 
            \State Modify $\x$ into $\xx$ while satisfying \autoref{eq:constraint}
            \Until $\xx \in b_{i}$ and $b_{i} \cup \xx \models$ Definition~\ref{def:actionfair}
        \EndWhile
    \EndFor
\EndFunction
\end{algorithmic}
\caption{Block-Based Recourse-Aware Re-ranking}
\label{alg:block-based}
\end{algorithm}

\subsection{Extensions}
\label{ss:extensions}
{\Ours} can be extended in a number of directions, which we briefly discuss here.
First, as alluded to in Section \ref{ss:preliminaries}, to cater for multi-class classifiers, one can stack several binary classifiers.
That is, fairness can be ensured across all outcomes by considering all possible one-versus-rest binary classifiers.

Second, our fairness definitions (and mechanism) can be generalized for more than two sub-groups.
For example, a company may want to consider fairness between under-18s, adults, and over-65s.
In this case, Definition \ref{def:repr} can be formalized such that $X$ only offers fair representation if $|p_X-p_i|\leq \epsilon$ for all protected sub-groups.
A revised Definition \ref{def:rankfair} follows naturally.
For Definition \ref{def:actionfair}, the redefined $r$ must consider fairness between multiple sub-groups.
For $n$ sub-groups, the group-level recourse fairness ratio is:
\begin{equation}
\textstyle
    r = \frac{\min(M_1, ..., M_n)}{\max(M_1, ..., M_n)}
\end{equation}

Finally, {\Ours} can also be applied to non-linear models (e.g., neural networks) when recourse-unaware methods are used to compute counterfactuals and initial rankings. 
Although methods are yet to be defined that can declare unfairness of recourse in non-linear settings~\cite{ustun2019actionable}, doing so, and incorporating recourse weights to determine these rankings, are interesting research challenges.

\section{Experimental Evaluation}
\label{s:experiments}
We evaluate {\Ours} using three real-world datasets.
We first examine 
the recourse fairness of both the initial data and the post-processed data using the current `fair classifiers'~(Section~\ref{ss:glrfa}).
We then study the effectiveness of {\Ours} at generating counterfactuals~(\ref{ss:expts-counterfactuals}) and in re-ranking~(\ref{ss:expts-rerank}), and compare them with competitive baselines.
This is followed by an analysis of {\Ours}-Block~(\ref{ss:expts-block}).

All code is written in Python 3 and we use Gurobi for the optimization problem of counterfactual generation. 
Experiments are conducted on macOS 10.15 with 2.4 GHz CPU and 8 GB RAM.

\subsection{Data}
\label{ss:data}
We apply {\Ours} to the challenge of automated customer credit and liability decision-making using the `German Credit' \cite{Dua:2019}, `Default of Credit Card Clients' (DC3) \cite{Dua:2019}, and HMDA \cite{HMDA} datasets.  
The German Credit dataset contains 1,000 records and has 20 features relating to individuals' financial and personal details, such as purpose of loan, missed payments, and marital status. 
A binary class label indicates individuals who are (un)successful in their loan application.
The DC3 dataset has 30,000 records and 24 features, and class labels indicating individuals who default on their payments. 
For both datasets, we convert categorical attributes
into actionable numerical attributes (e.g., length of employment) or binarize them (e.g., has guarantor), and aggregate some columns (e.g., months with zero balance) to construct more relevant features that can easily reflect infeasible recourse weights. 
After these modifications, the German Credit and DC3 datasets have 18 and 10 features, respectively. 
The HMDA dataset is based on the Home Mortgage Disclosure Act by the US government, and reports public loan data that can be filtered by year, geography, financial institution, and features such as the type/purpose/duration of loan, gender/race of applicants, or property type.
The original data comprises over 25M records, with 5M of these loans being clearly labeled as accepted/rejected. 
The 93 feature columns are filtered down to 26 where only the columns that reflect aggregated data are preserved.
This large HMDA dataset is only used for scalability experiments on {\Ours}, since none of the baselines can operate on this many records.
We assume that all attributes are independent of each other in these datasets.
We consider marital status as the protected attribute from literature for German Credit and DC3~\cite{kang2021multifair,bhargava2020limeout} and use gender for HMDA (as the marital status attribute is unavailable). 
As the classifier boundary (i.e., threshold) is defined based on the accept/reject labels in the real data, we only consider this boundary. 
Alternative thresholds are possible (as are other types of classifiers).

\subsection{Group-Level Recourse Fairness Analysis}
\label{ss:glrfa}
We first examine the effect of existing post-processing techniques, which are designed to improve fairness by re-classifying data points.
`Demographic parity' ensures that the acceptance rate is equal across sub-groups. 
`Equalized odds' (and its relaxed version, `equal opportunity') instead enforces fairness only among individuals who reach similar outcomes, by equalizing false positive/negative error rates or both (`weighted') \cite{hardt2016equality,pleiss2017fairness}.
While these methods do have different objectives, we include them to indicate how existing fairness definitions do not subsume the recourse-based definition and, as shown in Table \ref{tab:postproc-classifiers}, when these methods fail to improve recourse fairness, they actually exacerbate the problem significantly.
The `Recourse Cost' column demonstrates how {\Ours}'s aggregate weighted distance measure across sub-groups is lowered compared to the initial ranking, whereas measures that shift the classifier boundary worsen this distance.
This means that disadvantaged individuals under the post-processed classifier model would face even greater difficulty in improving their classification outcome. 
{\Ours}, however, increases recourse fairness by 15\% compared to the initial data, and up to 200\% compared to the alternatives, whilst also needing fewer points to be modified. 
{\Ours-Block} reduces the cost of interventions (difference in recourse cost) by 130\% compared to {\Ours}.
This guarantees that the re-ranked list remains closer in accuracy to the original ranked list, due to looser, less granular fairness constraints.

Table \ref{tab:postproc-classifiers} also includes results for two fair-ranking methods -- FA*IR~\cite{zehlike2017fa} and FoEiR~\cite{singh2018fairness}. 
FA*IR hurts recourse fairness and can only select up to 400 points, with {\Ours} attaining a ratio that is 37\% better. 
FoEiR can only re-rank up to 50 points due to inefficiency, but uses an exposure utility measure that can slightly improve recourse, although not as well as {\Ours}. 
Average results over 50 trials are presented for both.
The smaller German Credit data 
exhibited similar characteristics and is therefore omitted.

\begin{table}[t]
\centering
\small
\caption{\mbox{Recourse cost and fairness disparities (DC3)}}
\label{tab:postproc-classifiers}
\begin{tabular}{ccccc}
\toprule
\textbf{Post-Processed} & \mrb{2}{Sub-Group} & \textbf{\# Points} & \textbf{Recourse} & \mr{2}{$\bm{r}$} \\
\textbf{Data} & & \textbf{Changed} & \textbf{Cost} &  \\
\midrule
\mr{2}{Initial} & Single    & 0 & 7.688 & \multirow{2}{*}{0.759} \\
                & Married   & 0 & 5.837 &  \\
\midrule
Demographic     & Single    & 2771 & 6.023 & \multirow{2}{*}{0.472} \\
Parity          & Married   & 7 & 2.843 & \\
 \midrule
Equalized Odds  & Single    & 474 & 6.238 & \multirow{2}{*}{0.511} \\
(false negative)& Married   & 101 & 3.189 &  \\
\midrule
Equalized Odds  & Single    & 76 & 6.296 & \multirow{2}{*}{0.542} \\
(false positive)& Married   & 4910 & 3.412 &  \\
\midrule
Equalized Odds  & Single    & 76 & 6.296  & \multirow{2}{*}{0.422} \\
(weighted)      & Married   & 4944 & 2.658 &  \\
\midrule
\mr{2}{FA*IR} & Single & 0 & 7.611 & \multirow{2}{*}{0.554} \\
                & Married   & 57 & 4.217 & \\
\midrule
\mr{2}{FoEiR-DP}& Single & 12 & 6.788 & \multirow{2}{*}{0.793} \\
                & Married   & 7 & 5.383 & \\
\midrule
\mr{2}{{\Ours}}    & Single    & 135 & 6.663 &  \multirow{2}{*}{0.876} \\
                & Married   & 23 & 5.837 & \\
\midrule
{\Ours-Block}    & Single    & 18 & 6.891 &  \multirow{2}{*}{0.847} \\
($B = 25, \tau = \tfrac{1}{3}$)        & Married   & 49 & 5.837 & \\
\bottomrule
\end{tabular}
\vspace{-0.3cm}
\end{table}

\subsection{Counterfactual Analysis}
\label{ss:expts-counterfactuals}
Table~\ref{tab:cf} compares our explainable and efficient approach against two alternatives for generating counterfactuals. 
MACE~\cite{karimi2020model} uses formal verification techniques and satisfiability solvers to generate counterfactuals, whereas AR~\cite{ustun2019actionable} enumerates feasible ``flipsets'' (counterfactuals) for recommending and auditing recourse.
We present a qualitative comparison in terms of average closeness to the training data (Euclidean distance of counterfactual to nearest neighbor), average sparsity (number of attributes modified), and average runtime (per counterfactual generated)~\cite{verma2020counterfactual}. 
{\Ours} sometimes shows higher sparsity as different recourse weights may mean that smaller modifications are made to a higher number of attributes. 
However, our resulting counterfactuals are much closer to the original distribution of points, which is ultimately more important.
The efficiency of {\Ours} is comparable with that of AR, and both outperform MACE. 
Note that MACE and AR merely identify counterfactual points but do not modify the data, classification, or ranking. 
Hence, their influence on ranking fairness can only be measured in tandem with other methods for fair ranking.

\begin{table}[t]
\small
\setlength{\tabcolsep}{0.5\tabcolsep}
\caption{Comparison of counterfactual generation methods}
\label{tab:cf}
\begin{tabular}{c|ccc|ccc}
\toprule
\textbf{CF} & \multicolumn{3}{c|}{\textbf{German}} & \mcb{3}{\textbf{DC3}} \\
\textbf{Method} & \textbf{Close.} & \textbf{Spars.} & \textbf{Time (s)} & \textbf{Close.} & \textbf{Spars.} & \textbf{Time (s)} \\ 
\midrule
MACE    & 3095.47 & 2.23 & 1.713 & 2082.83 & 3.84 & 0.642 \\
AR      &    9.47 & 1.60 & 0.003 &  192.51 & 1.65 & 0.003 \\
{\Ours}-CF    &    1.42 & 1.36 & 0.003 &   72.04 & 2.68 & 0.004 \\
\bottomrule 
\end{tabular}
\vspace{-0.5cm}
\end{table}

\subsection{Re-Ranking Analysis}
\label{ss:expts-rerank}

We now consider {\Ours}'s effectiveness at re-ranking $\mathcal{D}$ to achieve improved recourse fairness. 
Before doing so, we introduce the user-specified `tolerance' parameter, $\tau$, which is used to determine the value for $\epsilon$, where $\epsilon = \tau \times p$. 
Our default setting is $\tau=\tfrac{1}{3}$.
For example, if 30\% of the database belongs to a protected sub-group and $\tau=\tfrac{1}{3}$, then $\epsilon = 30 \times \tfrac{1}{3} = 10\%$.
This means that the proportion of protected individuals in each block should be in the range 30$\pm$10\%.

We compare {\Ours} with state-of-the-art fair ranking baselines, where points are originally ranked according to their distance to counterfactuals, as computed by each of the counterfactual generation methods from Section \ref{ss:expts-counterfactuals}.
FA*IR~\cite{zehlike2017fa} selects the top-$k$ points ranked by relevance scores, to satisfy demographic parity without compromising on selection utility. FoEiR~\cite{singh2018fairness} balances exposure allocation in terms of three different fairness definitions (demographic parity, DP; disparate treatment, DT; and disparate impact, DI).
Since FA*IR was infeasible on larger datasets, we report the average values over 50 trials, sampling 400 points each time. 
Empirically, we find that setting $\alpha$, which is a FA*IR-specific parameter, to 0.15 offers the best results and we use this setting throughout.
Similarly, FoEiR could only handle 50 points, and so we similarly obtain average results after 50 trials.
In comparison, {\Ours} runs efficiently (<2.5 seconds) on both full datasets, and it is 4x faster than FoEiR on its smaller sample.

We evaluate these results with respect to different ranking quality metrics (normalized discounted KL-divergence, rKL; normalized discounted difference, rND; normalized discounted ratio, rRD)~\cite{yang2017measuring}. 
rND indicates the difference between the protected sub-group's proportion among top-$k$ records and overall, while rKL uses Kullback-Leibler divergence for the expectation of this difference. 
rRD computes a similar difference to rND but between the ratios of the (minority) protected sub-group to the majority.

\begin{table}[t]
\setlength{\tabcolsep}{0.9\tabcolsep}
\caption{Comparison of fair ranking methods}
\label{tab:rank}
\small
\begin{tabular}{cc|ccc|ccc}
\toprule
\multicolumn{2}{c|}{\textbf{Method for...}} & \multicolumn{3}{c|}{\textbf{German}} & \multicolumn{3}{c}{\textbf{DC3}} \\
\textbf{CF} & \textbf{Ranking} & \textbf{rKL} & \textbf{rND} & \textbf{rRD} & \textbf{rKL} & \textbf{rND} & \textbf{rRD}\\ 
\midrule
\multirow{5}{*}[-1ex]{\rotatebox[origin=c]{90}{\textbf{MACE}}}
& Initial  & 0.051 & 0.194 & \textbf{0.000} & 0.129 & \textbf{0.248} & \textbf{0.042} \\
& FA*IR    & 0.051 & 0.196 & \textbf{0.000} & 0.124 & 0.249 & \textbf{0.042} \\
& FoEiR-DP & 0.049 & 0.188 & 0.009 & 0.070 & 0.229 & 0.061 \\
& FoEiR-DT & 0.049 & 0.188 & 0.009 & \textbf{0.069} & 0.229 & 0.061 \\
& FoEiR-DI & 0.048 & 0.186 & 0.009 & \textbf{0.069} & 0.229 & 0.061 \\
& {\Ours-Rk}   & \textbf{0.048} & \textbf{0.185} & \textbf{0.000} & 0.129 & \textbf{0.248} & \textbf{0.042} \\
\midrule
\multirow{5}{*}[-1ex]{\rotatebox[origin=c]{90}{\textbf{AR}}}
& Initial  & 0.054 & 0.223 & \textbf{0.000} & \textbf{0.013} & 0.112 & 0.223 \\
& FA*IR    & 0.066 & 0.265 & \textbf{0.000} & 0.021 & 0.125 & \textbf{0.205} \\
& FoEiR-DP & 0.080 & 0.231 & 0.035 & 0.038 & 0.183 & 0.246 \\
& FoEiR-DT & 0.080 & 0.232 & 0.035 & 0.038 & 0.183 & 0.246 \\
& FoEiR-DI & 0.079 & 0.229 & 0.039 & 0.036 & 0.185 & 0.247 \\
& {\Ours-Rk}   & \textbf{0.051} & \textbf{0.218} & \textbf{0.000} & \textbf{0.013} & \textbf{0.111} & 0.221 \\
\midrule
\multirow{5}{*}[-1.5ex]{\rotatebox[origin=c]{90}{\textbf{{\Ours}-CF}}}
& Initial  & 0.026 & 0.085 & 0.121 & 0.025 & \textbf{0.153} & 0.299 \\
& FA*IR    & 0.118 & 0.316 & 0.033 & 0.026 & 0.148 & \textbf{0.262} \\
& FoEiR-DP & 0.052 & 0.165 & 0.170 & 0.039 & 0.190 & 0.276 \\
& FoEiR-DT & 0.052 & 0.165 & 0.170 & 0.041 & 0.190 & 0.276 \\
& FoEiR-DI & 0.052 & 0.172 & 0.168 & 0.040 & 0.191 & 0.273 \\
& {\Ours-Rk}   & \textbf{0.005} & \textbf{0.047} & \textbf{0.032} & \textbf{0.023} & \textbf{0.153} & 0.298 \\
\bottomrule
\end{tabular}
\end{table}


Table~\ref{tab:rank} shows that {\Ours} generally outperforms the alternatives on both datasets,
consistently maintaining or improving these ranking metrics compared to the initial ranked list.
FoEiR generally achieves better scores with the recourse-unaware MACE, while FA*IR performs comparatively well with AR. 
{\Ours} allows any combination of techniques to be applied for re-ranking.


\begin{figure}[t]
    \centering
     \begin{subfigure}[b]{\columnwidth}
        \centering
        \includegraphics[height = .5cm]{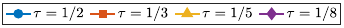}
    \end{subfigure}
    \\
    \vspace{0.1cm}
    \begin{subfigure}[b]{0.49\columnwidth}
        \centering
        \includegraphics[width = \textwidth]{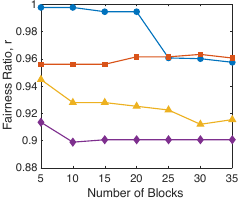}
        \caption{German Credit dataset}
        \label{fig:german-blocks}
    \end{subfigure}
    \hfill
    \begin{subfigure}[b]{0.49\columnwidth}
        \centering
        \includegraphics[width = \textwidth]{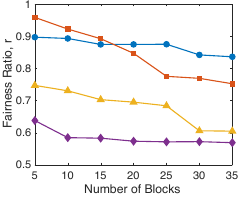}
        \caption{DC3 dataset}
        \label{fig:dc3-blocks}
    \end{subfigure}
    \vspace{-.2cm}
    \caption{Variation of group-level recourse fairness with $\bm{\tau}$ and $\bm{B}$. Note the false origins in both plots.}
    \label{fig:plots-blocks}
    \vspace{-.2cm}
\end{figure}

\subsection{Block-Based Re-Ranking Analysis}
\label{ss:expts-block}
We now examine the effectiveness of performing block-based re-ranking by considering different granularities and strictness. 
We assume all blocks are equal in size with block size being governed by the number of blocks into which the ranked list is sub-divided.
%

Figure~\ref{fig:plots-blocks} shows how group-level recourse fairness changes with the number of blocks and the tolerance.
Recourse fairness decreases as the number of blocks increases as there is a more restricted range of values into which attributes can be perturbed for re-ranking. 
{\Ours}-Block is relatively stable to an increase in $B$, although there is an expected decrease in $r$ as it becomes harder to satisfy the fairness constraints for small blocks. 
Similarly, while low $\tau$-values lead to recourse fairness being more strictly enforced within each block, a very low $\tau$-value results in increased unfairness by forcing corrections to more blocks (especially those encountered early). 

We also study the effectiveness of our re-ranking approach by considering the number of unfair blocks that are transformed into fair blocks through re-ranking (Table~\ref{tab:effectiveness}). 
As expected, with high tolerance and very few (large) blocks, the dataset is often already deemed to be fair or it is easily made fair. 
As the number of blocks is increased and each block thereby becomes smaller, re-ranking has a greater impact. 
Notably, it is more effective to reduce the size of blocks than to reduce the tolerance, as a lower tolerance leads {\Ours}-Block to be less likely to create fair blocks.

Figure \ref{fig:married-plots} shows the proportion of the two sub-groups in each block before and after re-ranking ($B = 20, \tau = \tfrac{1}{8}$, DC3).
The allowable tolerance margin around the ideal proportion, based on the overall dataset, is also indicated.
Given the tolerance bound, the original dataset shows that nine of the 20 blocks are unfair initially with only five blocks remaining unfair after re-ranking. 
This highlights the impact of using a stricter tolerance as more unfair blocks remain, compared to that at lower $\tau$ values. 
However, these blocks become concentrated at the bottom of the ranking, with the last block in particular becoming more unfair as there is no opportunity for its fairness to be corrected.
With a looser tolerance, earlier blocks containing slightly unfair representations can be ignored during re-ranking, which prevents later blocks from accumulating highly imbalanced outliers.

In terms of runtime, when $B$ is low and/or $\tau$ is large, {\Ours}-Block runs in less than two seconds on the larger DC3 dataset.
As $B$ increases and/or $\tau$ decreases, runtime increases by up to nine times -- 14.3s ($B = 35, \tau = \tfrac{1}{5}$) vs. 1.7s ($B = 5, \tau = \tfrac{1}{2}$).
This is to be expected as the mechanism needs to make more satisfiability checks due to the tighter constraints and the larger number of blocks.
Moreover, {\Ours}-Block is nearly 30x faster, while changing the fairness ratio by no more than 10\% (up to $B = 25$ and $\tau = \tfrac{1}{3}$ on DC3).
\begin{table}[t]
\centering
\small
\caption{Effectiveness of {\Ours}-Block (DC3); initial and final number of unfair blocks shown; number in brackets denotes number of blocks made fair}
\label{tab:effectiveness}
\begin{tabular}{c|cccc}
\toprule
\multirow{2}{*}{$\bm{B}$} & \multicolumn{4}{c}{\textbf{Tolerance, $\bm{\tau}$}}\\
& \textbf{1/2} & \textbf{1/3} & \textbf{1/5} & \textbf{1/8} \\
\midrule
 5 & $0\rightarrow0$  (0) & $1\rightarrow0$    (1) & $2\rightarrow1$    (1) &   $1\rightarrow1$  (0)\\

10 & $0\rightarrow0$  (0) & $3\rightarrow0$    (3) & $5\rightarrow1$    (4) &   $2\rightarrow2$  (0)\\

15 & $2\rightarrow0$  (2) & $6\rightarrow1$    (5) & $8\rightarrow2$    (6) &   $5\rightarrow4$  (1)\\

20 & $2\rightarrow0$  (2) & $8\rightarrow1$    (7) & $10\rightarrow2$   (8) &   $9\rightarrow5$  (4)\\

25 & $3\rightarrow0$  (3) & $13\rightarrow1$  (12) & $12\rightarrow3$   (9) & $14\rightarrow7$  (7)\\

30 & $6\rightarrow0$  (6) & $13\rightarrow1$  (12) & $14\rightarrow3$  (11) & $13\rightarrow9$  (4)\\

35 & $7\rightarrow0$  (7) & $14\rightarrow1$  (13) & $19\rightarrow5$  (14) & $21\rightarrow11$  (10)\\
\bottomrule
\end{tabular}
\vspace{-0.0cm}
\end{table}

\begin{figure}[tb]
    \centering
     \begin{subfigure}[b]{0.8\columnwidth}
        \centering
        \includegraphics[width = \textwidth]{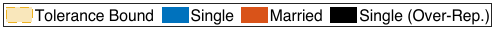}
    \end{subfigure}
    \\
    \begin{subfigure}[b]{0.49\columnwidth}
        \centering
        \includegraphics[height = 4cm]{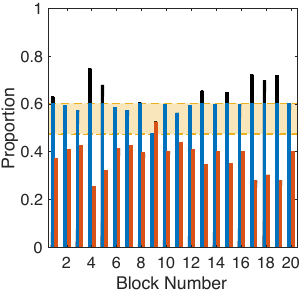}
        \caption{Before}
        \label{fig:married-pre}
    \end{subfigure}
    \hfill
    \begin{subfigure}[b]{0.47\columnwidth}
        \centering
        \includegraphics[height = 4cm]{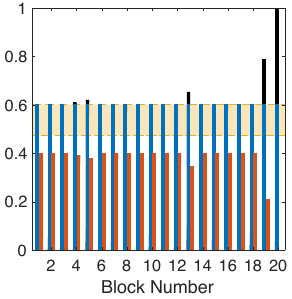}
        \caption{After}
        \label{fig:married-post}
    \end{subfigure}
    \vspace{-0.3cm}
    \caption{Group proportions before and after re-ranking}
    \label{fig:married-plots}
    \vspace{-.3cm}
\end{figure}

\subsection{Scalability Analysis}
\label{ss:scalability}
Finally, we demonstrate that {\Ours}-Block runs efficiently on large-scale data such as the HMDA dataset, with performance results that remain stable even as the number of records increases.
None of the re-ranking alternatives (e.g., FA*IR and FoEiR) could handle large datasets, so they are excluded here.
We sample between 2,500 and 1M records HMDA dataset to maintain an initial fairness ratio of 0.83.
We run {\Ours}-Block using a fixed block size of 100 records (i.e., the total number of blocks is variable) on these samples, presenting results averaged over 10 trials. Figure~\ref{fig:intial-final} shows that the final fairness ratio attained is at least 0.95 and does not exhibit much variability as dataset size increases. 
The ranking quality metrics (Figure~\ref{fig:eval-measures}), which are high when the number of records is small, quickly decrease and remain low for larger datasets.

\begin{figure}[t]
    \centering
    \begin{subfigure}[b]{0.49\columnwidth}
        \centering
        \includegraphics[width = \textwidth]{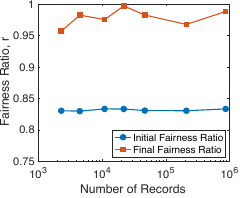}
        \caption{Fairness Ratio}
        \label{fig:intial-final}
    \end{subfigure}
    \hfill
    \begin{subfigure}[b]{0.49\columnwidth}
        \centering
        \includegraphics[width = \textwidth]{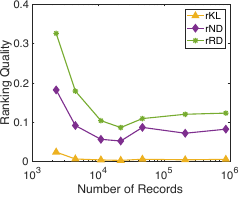}
        \caption{rKL, rND, and rRD}
        \label{fig:eval-measures}
    \end{subfigure}
    \vspace{-.2cm}
    \caption{Effect on recourse fairness and ranking quality as the number of records in the dataset varies}
    \label{fig:varying-rows}
\vspace{-0.1cm}
\end{figure}

These results show that our solution can easily handle any arbitrarily large dataset that is sub-divided into smaller blocks within which group-level recourse fairness is being enforced. 
Our method essentially uses sliding window processing, as a block only requires the next adjoining block of candidate records for {\Ours}-Block to improve its fairness (see Algorithm~\ref{alg:block-based} and Section \ref{sss:block-based-rerank}).

\section{Conclusion}

{\Ours} is a solution for improving group-level fairness in ranked lists with respect to both representation- \textit{and} recourse-based constraints.
{\Ours} performs iterative computations to identify feasible recourse actions that require minimal cost and can result in a fairer ranked list. 
A block-based extension can handle adjustable granularities and multiple target points or boundaries.
While these new recourse-based approaches have recently been considered more in the context of societal fairness where appropriate policy measures are in place for disadvantaged groups, they are more broadly applicable to technical optimization tasks where fairness is a constraint or part of the objective, such as resource allocation in computer systems or distributing funding to organizations. 
We leave these applications to future work.

\begin{acks}
This work is supported in part by the UK Engineering and Physical Sciences Research Council under Grant No. EP/L016400/1.
Aparajita is supported via a Feuer International Scholarship in Artificial Intelligence.
We thank Efehan Madran, Aaron MacFarlane, and Amina Tkhamokova for their valuable contributions.
\end{acks}

\clearpage
\balance
\bibliographystyle{ACM-Reference-Format}
\bibliography{sample-base}


\end{document}